\documentclass{article}
\usepackage{amsmath}
\usepackage{amsthm}
\usepackage{multirow}

\newcommand{\vect}[1]{\boldsymbol{\mathbf{#1}}} 

 \usepackage[preprint]{neurips_2025}

\usepackage[utf8]{inputenc} 
\usepackage{hyperref}       
\usepackage{url}            
\usepackage{booktabs}       
\usepackage{amsfonts}       
\usepackage{nicefrac}       
\usepackage{microtype}      
\usepackage{xcolor}         
\usepackage{algorithm}
\usepackage{algpseudocode}
\usepackage{amsmath}
\usepackage{graphicx}
\usepackage{cleveref}
\usepackage{tablefootnote}
\usepackage{subcaption}

\newcommand{\MMAD}{\textbf{MMAD}}
\newcommand{\MTEone}{\textbf{MTE1}}
\newcommand{\MTEtwo}{\textbf{MTE2}}
\newcommand{\FixP}{\textbf{FixP}}
\newcommand{\singleM}{singleM}
\newcommand{\singleN}{singleN}
\newcommand{\singleK}{singleK}
\newcommand{\baseM}{baseM}
\newcommand{\baseN}{baseN}
\newcommand{\baseK}{baseK}
\newcommand{\Cone}{[C_1]}
\newcommand{\Ctwo}{[C_2]}
\newcommand{\Cn}{[C_n]}
\newcommand{\Vone}{[V_1]}
\newcommand{\Vtwo}{[V_2]}
\newcommand{\Vn}{[V_n]}
\newcommand{\C}[1]{[C_{#1}]}
\newcommand{\V}[1]{[V_{#1}]}

\newtheorem{theorem}{Theorem}[section]

\newtheorem{lemma}{Lemma}[section]
\newtheorem{remark}{Remark}[section]
\newtheorem{example}{Example}[section]

\crefname{algorithm}{Algorithm}{Algorithm}
\crefname{lemma}{Lemma}{Lemma}
\crefname{table}{Table}{Table}
\crefname{theorem}{Theorem}{Theorem}
\crefname{corollary}{Corollary}{Corollary}
\crefname{equation}{Eq.}{Eq.}
\crefname{figure}{Fig.}{Fig.}
\crefname{section}{Section}{Section}
\title{AMLA: MUL by ADD in FlashAttention Rescaling}

\makeatletter
\let\oldthanks\thanks
\renewcommand{\thanks}[1]{%
    \oldthanks{#1}%
    \addtocounter{footnote}{-1}
}
\makeatother

\author{
  Qichen Liao\\
  \texttt{liaoqichen2@huawei.com} \\
  \And
  Chengqiu Hu \\
  \texttt{hu.chengqiu@huawei.com} \\
  \And
  Fangzheng Miao \\
  \texttt{miaofangzheng@huawei.com} \\
  \AND
  Bao Li \\
  \texttt{li.bao@huawei.com} \\
  \And
  Yiyang Liu \\
  \texttt{liuyiyang16@huawei.com} \\
  \And
  Junlong Lyu \\
  \texttt{lyujunlong@huawei.com } \\
  \AND
  Lirui Jiang \\
  \texttt{jianglirui1@huawei.com} \\
  \And
  Jun Wang \\
  \texttt{hwjun.wang@huawei.com} \\
  \And
  Lingchao Zheng \\
  \texttt{zhenglingchao@huawei.com} \\
  \AND
  Jun Li \\
  \texttt{lijun276@huawei.com} \\
  \And
  Yuwei Fan\thanks{Corresponding author} \\
  \texttt{fanyuwei2@huawei.com} \\
  \AND
  \textit{Huawei}
}

\begin{document}

\maketitle

\begin{abstract}

Multi-head Latent Attention (MLA) significantly reduces KVCache memory usage in Large Language Models while introducing substantial computational overhead and intermediate variable expansion. This poses challenges for efficient hardware implementation --- especially during the decode phase. This paper introduces Ascend MLA (AMLA), a high-performance kernel specifically optimized for Huawei's Ascend NPUs. AMLA is built on two core innovations: (1) A novel FlashAttention-based algorithm that replaces floating-point multiplications with integer additions for output block rescaling, leveraging binary correspondence between FP32 and INT32 representations; (2) A Preload Pipeline strategy with hierarchical tiling that maximizes FLOPS utilization: the Preload Pipeline achieves Cube-bound performance, while hierarchical tiling overlaps data movement and computation within the Cube core. Experiments show that on Ascend 910 NPUs (integrated in CloudMatrix384~\cite{HisiCM384}), AMLA achieves up to 614 TFLOPS, reaching \textbf{86.8\%} of the theoretical maximum FLOPS, outperforming the state-of-the-art open-source FlashMLA implementation, whose FLOPS utilization is up to 66.7\%\footnote{The tensor core utilization of FlashMLA is 66.7\% of the maximum theoretical peak of H800 SXM5, and 80\% of the throttled theoretical peak.} on NVIDIA H800 SXM5~\cite{FlashMLA}. The AMLA kernel has been integrated into Huawei's CANN and will be released soon.

\end{abstract}

\section{Introduction}

Attention mechanisms lie at the heart of the Transformer architecture~\cite{vaswani2017attention}, empowering Large Language Models (LLMs) to model long-range dependencies through explicit token-to-token interactions. As downstream applications increasingly demand ultra-long context windows~\cite{yao-2023,libo2024large,yiwei2024large,yuqi2024a}, the computational cost and memory footprint of attention have become critical bottlenecks. While Multi-Query Attention (MQA)~\cite{mqa} and Grouped-Query Attention (GQA)~\cite{gqa} reduce memory pressure via key-value head sharing, Multi-Head Latent Attention (MLA)~\cite{liu2024deepseek} pushes compression further by projecting key-value representations into a shared low-dimensional latent space, preserving model quality while drastically shrinking the KVCache.

However, this efficiency comes at a cost: MLA shifts the computational profile from memory-intensive to compute-intensive. As shown in \cref{fig:attentions}, while Multi-Head Attention (MHA) and GQA remain constrained by memory bandwidth due to low arithmetic intensity (FLOPS/Byte), MLA's latent sharing significantly increases compute intensity, making it ideal for AI accelerators, especially when combined with Multi-Token Prediction (MTP), which further amplifies compute demand. This synergy makes MLA particularly well-suited for high-performance NPUs like the Ascend 910 which is designed for high FLOPS density and energy efficiency. 
\begin{figure}
    \centering
    \includegraphics[width=0.7\linewidth]{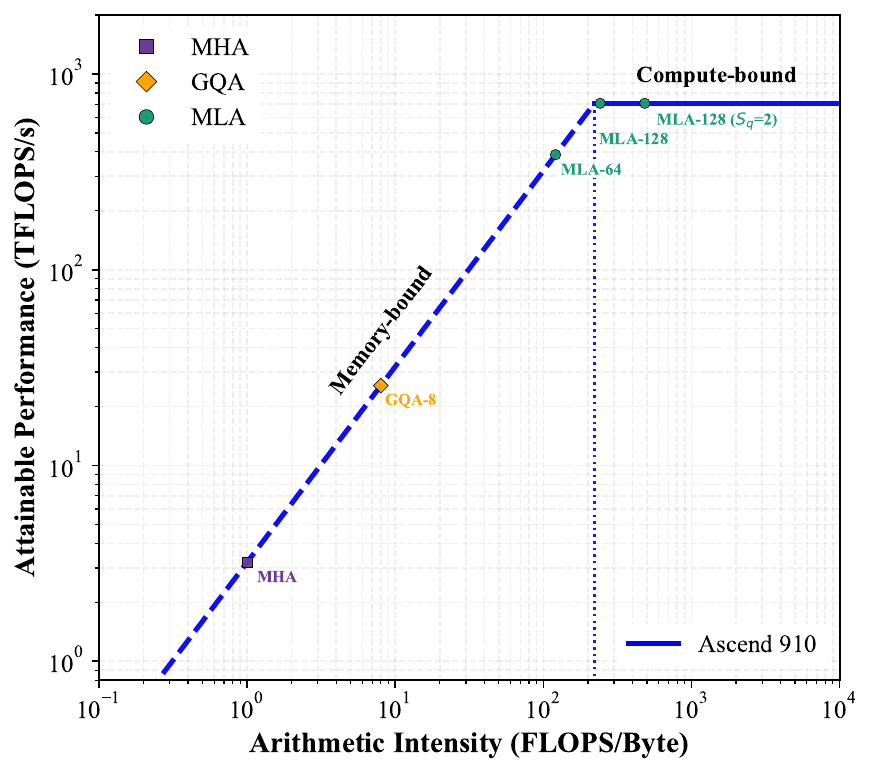}
    \caption{Roofline Analysis of BF16 Decoding on Ascend 910. The dashed segment indicates the region where performance is limited by memory bandwidth, whereas the horizontal solid lines represent the peak compute-bound performance achievable. Data points corresponding to different attention variants are plotted, showcasing their operational regimes and proximity to the hardware limits.}
    \label{fig:attentions}
\end{figure}

Yet, naive MLA implementations fail to exploit the Ascend 910's full potential. Two key bottlenecks emerge:

\begin{enumerate}

    \item \textbf{Large Output Tensors}.  
    Following the Flash Attention algorithm \cite{FlashAttention-1}, the output $\vect{O}_i$ is rescaled frequently. However, the rescaling of output $\vect{O}_i \leftarrow \vect{O}_{i-1} \cdot \exp(m_{i-1}-m_i) + \vect{P}_i\vect{V}_i$ generates large intermediate tensors (e.g., $\vect{O} \in \mathbb{R}^{128 \times 512}$ in FP32). Due to the large size of the output matrix $\vect{O}$, these tensors cannot fit within the register file of a single Streaming Multiprocessor (SM) on GPU architectures. For the Ascend 910, $\vect{O}$ must be repeatedly shuttled between Global Memory (GM) and Unified Buffer (UB) --- GM and UB are storage structures at different levels.

    \item \textbf{Underutilized Heterogeneous Compute}. Current kernels serialize Cube (attention score matmul) and Vector (softmax, scaling) operations, leaving cores idle. The lack of pipelining between physically separated compute units results in suboptimal FLOPS utilization, which is a critical waste on hardware designed for parallelism.
\end{enumerate}

Although prior work, such as FlashMLA~\cite{FlashMLA}, mitigates the bottleneck via reducing the row number of blocks and employing a complex scheduling algorithm, achieving 660 TFLOPS in compute-bound configuration on H800 SXM5. However, such an approach inevitably introduces additional overhead due to the repetitive movement and management of KVCache. 

To address all the above bottlenecks, we propose \textbf{Ascend MLA (AMLA)}, a co-designed algorithm and kernel implementation that aligns MLA's compute patterns with the Ascend 910's architectural strengths. Our contributions are:
\begin{itemize}
\item \textbf{Replacing Multiplication with Addition}: We reformulate the output update using binary reinterpretation of FP32 as INT32, replacing floating-point multiplications with atomic integer additions. This enables \textit{in-place} updates directly in GM, eliminating data movement between GM and UB for intermediate tensors.

\item \textbf{Preload Pipeline \& Hierarchical Tiling}: We introduce a \textit{Preload Pipeline} that overlaps execution of Cube cores (matrix multiplication) and Vector cores (softmax, rescaling), ensuring the kernel remains Cube-bound. Besides, we design a \textit{Hierarchical Tiling} strategy within the Cube stages, overlapping data movement with computation. This dual-layer optimization collectively minimizes pipeline bubbles and sustains near-peak compute utilization.
\end{itemize}
AMLA achieves \textbf{86.8\% FLOPS utilization} on Ascend 910, surpassing FlashMLA's 66.7\% on H800 SXM5~\cite{FlashMLA} while preserving numerical precision. The kernel is integrated into Huawei's CANN and will be open-sourced.

\noindent\textbf{Paper Structure}: Section 2 reviews attention variants and the Ascend 910 architecture. Section 3 details AMLA's algorithm and numerical analysis. Section 4 describes a hardware-aware implementation. Section 5 presents experiments. Section 6 concludes.
\section{Background}

\subsection{Attention and FlashAttention}

The Transformer architecture~\cite{vaswani2017attention} has become the foundation of modern LLMs, including DeepSeek~\cite{liu2024deepseek,deepseekai2024deepseekv3,deepseekai2025deepseekr1}, Qwen~\cite{Qwen2023_TechnicalReport,wu2025qwenimagetechnicalreport}, ChatGPT~\cite{gpt3,gpt4}, and Gemini~\cite{Gemini2023_Model}. Its defining innovation, self-attention, enables dynamic modeling of contextual relationships across input tokens, surpassing the fixed receptive fields of CNNs and sequential dependencies of RNNs. For each attention head, query ($\vect{Q}$), key ($\vect{K}$), and value ($\vect{V}$) matrices are computed, and attention output is given by:

\begin{equation}\label{eq:atten}
    \text{Attention}(\vect{Q}, \vect{K}, \vect{V}) = \text{softmax}\left(\frac{\vect{Q} \vect{K}^T}{\sqrt{D_k}}\right)\vect{V},
\end{equation}

where $D_k$ denotes the head dimension of $\vect{Q}$ and $\vect{K}$.

While theoretically efficient, the naive implementation of \cref{eq:atten} suffers from low hardware utilization due to excessive data movement between high-bandwidth memory (HBM) and compute units. To address this, FlashAttention~\cite{FlashAttention-1, FlashAttention-2} introduces tiling, recomputation, and online softmax, reducing memory traffic and achieving over 5$\times$ speedup on GPUs. Algorithmic details are provided in \cref{alg:flash_attention}.

As context lengths and batch sizes grow during LLM inference, KVCache memory consumption becomes a critical scalability constraint. To mitigate this, architectures such as MQA~\cite{mqa} and GQA~\cite{gqa} have been proposed. MQA shares a single $\vect{K}/\vect{V}$ head across all query heads, drastically reducing KVCache size. GQA extends this idea by grouping query heads and assigning each group a dedicated $\vect{K}/\vect{V}$ head, striking a balance between memory efficiency and model expressiveness.

\subsection{Multi-Head Latent Attention}

\textbf{MLA}~\cite{liu2024deepseek} further compresses KVCache by projecting key-value representations into a shared low-dimensional latent space. Specifically, the original $[\vect{K}, \vect{V}] \in \mathbb{R}^{S_2 \times (N_2 \times (D_k + D_v))}$ is compressed into a latent vector $c \in \mathbb{R}^{S_2 \times D_c}$, where 
$N_2$ is the head number of the original $\vect{K}/\vect{V}$, $S_2$ is the context length of KVCache, $D_k$ and $D_v$ are the head dimensions of $\vect{K}$ and $\vect{V}$, respectively. $D_c \ll N_2(D_k + D_v)$ is the dimension of the latent vector $c$. This compact $c$ is cached instead of full $\vect{K}/\vect{V}$ tensors.

At each decoding step, $c$ is up-projected to restore full-rank representations:
\begin{equation*}
\vect{K} = c \vect{W}_k,\quad \vect{V} = c \vect{W}_v.
\end{equation*}

By precomputing $\vect{Q}' = \vect{Q} \vect{W}_k^T$, the attention score becomes $\vect{Q}'c^T$, avoiding explicit $\vect{Q}\vect{K}^T$ computation. Similarly, $\vect{W}_v$ can be fused into the output stage. Since $c$ is shared across all heads, MLA achieves MQA-level memory compression while preserving multi-head expressivity. Implementation details, including Rotary Position Embedding (RoPE) integration, are elaborated in~\cite{liu2024deepseek}.

Notably, MLA shifts the computational profile from memory-bound to compute-bound, which is a characteristic that aligns well with high-throughput accelerators like the Ascend 910.

\subsection{Ascend 910 Architecture}

The Ascend 910 NPU is a dual-die system, with each die implementing Huawei’s Da Vinci V220 architecture~\cite{davinci220} (see \cref{fig:davinci}). Each die integrates 24 \textit{Cube cores} optimized for high-throughput matrix operations and 48 \textit{Vector cores} tailored for element-wise and reduction operations. 

\begin{figure}
    \centering
    \includegraphics[width=0.8\linewidth]{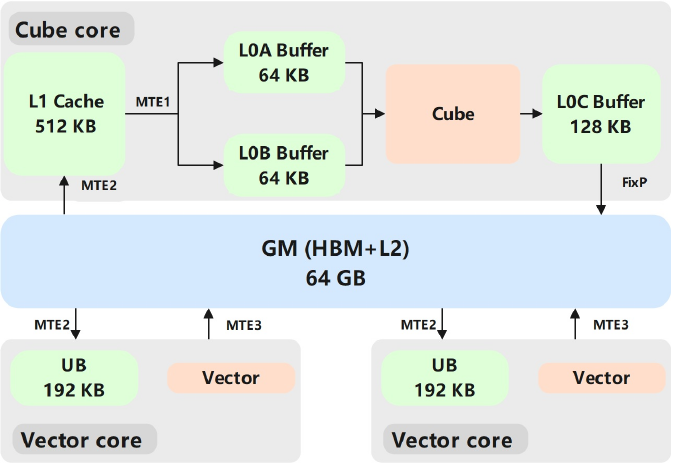}
    \caption{Da Vinci V220 architecture. Cache capacities are accessible via the Ascend C API \emph{GetCoreMemSize}.}
    \label{fig:davinci}
\end{figure}

{\bf Hierarchical Memory Architecture.} Each die features 64 GB HBM and 192 MB L2 cache, delivering 1.6 TB/s HBM bandwidth per die (3.2 TB/s aggregate). Each Cube core includes a 512 KB L1 cache and a partitioned 256 KB L0 cache (64 KB input buffers L0A/L0B, 128 KB output buffer L0C). Each Vector core utilizes a 192 KB UB as its primary working memory.

\begin{table}[htbp]
    \centering
    \begin{tabular}{cccccc}
        \toprule
        \multicolumn{2}{c}{\textbf{Component}} & \textbf{Capacity} & \textbf{Bandwidth}\\
        \midrule
        \multicolumn{2}{c}{HBM} & 64 GB & 1.6 TB/s \\
        \midrule
        \multirow{5}{*}{Cube core} & \multirow{2}{*}{L1 cache} & \multirow{2}{*}{512 KB} & --- \\
        & & & --- \\
       \cmidrule{2-4}
        & L0A cache & 64 KB& --- \\
        & L0B cache & 64 KB& ---\\
        & L0C cache & 128 KB& ---\\
        \midrule
        \multirow{1}{*}{Vector core} & UB & 192 KB & --- \\
        \bottomrule
    \end{tabular}
    \caption{Hierarchical memory system for each die in Ascend 910.}
    \label{tab:hier mem}
\end{table}

{\bf Cube-Vector Core Collaboration Model.}  
The Cube and Vector cores operate with physically distinct memory spaces. Data exchange between them requires explicit movement through GM, encompassing HBM and L2 cache. This design prioritizes core specialization and scalability, enabling concurrent execution, but requires careful orchestration in kernel design to minimize GM traffic and maximize compute overlap.

\subsection{Arithmetic Intensity of Different Attentions} \label{subsec:arithmeticintensity}

Arithmetic intensity is a fundamental metric in computer architecture, defined as the ratio of FLOPS to memory access (in bytes). It serves as a critical indicator for determining whether a workload is compute-bound or memory-bound. During the decode phase, tasks with high arithmetic intensity, such as MLA, are typically compute-bound, meaning their performance is constrained by the computational throughput of the accelerator. In contrast, tasks with low arithmetic intensity, such as MHA, are memory-bound, limited primarily by memory bandwidth.

The arithmetic intensity is formally expressed as:
\[
\text{Arithmetic Intensity} = \frac{\text{FLOPS}}{\text{Memory Access (Bytes)}}.
\]

For standard attention mechanisms, the FLOPS and memory traffic can be derived as follows:
\[
\text{FLOPS} = 2N_1 S_1 S_2 (D_k + D_v), \quad
\text{MEM}_{KV} = 
\begin{cases}
2N_2 S_2 (D_k + D_v), & \text{for MHA/GQA;} \\
2S_2 D_k,                         & \text{for MLA.}
\end{cases}
\]
Here, $N_1$ and $N_2$ denote the number of query heads and key/value heads. $S_1$ and $S_2$ represent the context length of query and key/value, where $S_1$ is usually 1 in the decode phase, and might be a bit larger with MTP. FLOPS account for both multiplication and addition operations, and memory is measured in bytes (e.g., BF16 parameters occupy 2 bytes each).

Substituting these expressions, the arithmetic intensity becomes:
\[
\text{Arithmetic Intensity} = 
\begin{cases}
N_1S_1,                  & \text{for MHA/GQA;} \\
N_1S_1\frac{D_k + D_v}{D_k}, & \text{for MLA.}
\end{cases}
\]

Different attention mechanisms thus exhibit distinct arithmetic intensity profiles. We summarize the arithmetic intensities of MHA, GQA, and MLA under common configurations in \cref{tab:arithmetic intensity} and visualize their characteristics in \cref{fig:attentions}. For further discussion on hardware-efficient attention designs, see \cite{HardwareEfficientAttention}.

During the decode phase, attention mechanisms exhibit distinct arithmetic intensities (FLOPS / Byte), leading to divergent hardware utilization profiles. Different hardware-efficient variants of attention are proposed to utilize modern hardware effectively~\cite{HardwareEfficientAttention}. As shown in \Cref{fig:attentions}, MHA remains memory-bound due to low arithmetic intensity. GQA improves intensity slightly but still favors memory bandwidth. In contrast, MLA, by sharing latent representations across heads, significantly increases arithmetic intensity, making it compute-bound.

This shift renders MLA particularly well-suited for architectures with high compute density and moderate memory bandwidth --- such as the Ascend 910 and NVIDIA H800. When combined with MTP, MLA’s compute demand further escalates, reinforcing the need for hardware-aware optimization.

To quantify efficiency, we adopt \textbf{FLOPS Utilization (FU)} as our primary performance metric, measuring the percentage of peak theoretical compute throughput achieved during kernel execution.

{\renewcommand{\arraystretch}{1.5}
\begin{table}[]
    \centering
    \begin{tabular}{cccccccc}
    \hline
         \textbf{Attention Variant} & MHA & GQA & MLA-64 & MLA-128 & MLA-128($S_q=2$)\\
         \textbf{$Q_{head}$} & 64 & 64 & 64 & 128 & 128\\
         \textbf{$K_{head}(V_{head})$} & 64 & 8 & 1 & 1 & 1 \\
         \textbf{$S_q$} & 1 & 1 & 1 & 1 & 2\\
         \textbf{Arithmetic Intensity} & $ 1 $& $8$ & $\approx 121$ & $\approx 242$ & $\approx 484$ \\
         \hline
    \end{tabular}
    \vspace{0.2cm}
    \caption{Arithmetic Intensity (FLOPS/Bytes) of Different Attentions.}
    \label{tab:arithmetic intensity}
\end{table}
}

\subsection{FlashMLA}

Proposed by DeepSeek, FlashMLA is an optimized MLA decoding kernel designed for NVIDIA Hopper GPUs~\cite{FlashMLA}. To deal with the significant space occupancy of the output matrix, FlashMLA sets $\text{BLOCK\_SIZE\_M}$ to be 64, which means that each FlashAttention iteration processes the computation of 64 rows. For cases where the group size exceeds 128, this approach may introduce additional data movement.

Furthermore, since the capacity of registers for each SM is 256KB, which is exactly twice $64\times512\times4=\text{128KB}$, it is infeasible to enable Tensor core and Cuda core at the same time if rescaling of a block is performed all at once. FlashMLA splits along the column direction and implements a sophisticated "seesaw" scheduling algorithm, enabling compute-bound optimization by overlapping CUDA Core and Tensor Core operations.

Overall, by employing smaller block granularity and a complex scheduling algorithm, FlashMLA addresses the challenges of implementing the MLA kernel and achieves good performance on the H800 SXM5.

\section{Algorithm}
\label{sec:Algorithm}

In this section, we present the AMLA algorithm and analyze its numerical accuracy and hardware efficiency.

\subsection{Motivation}

We begin by reviewing the original FlashAttention algorithm \cite{FlashAttention-1,FlashAttention-2} (\cref{alg:flash_attention}), which FlashMLA also adopts.

\begin{algorithm}[]
\caption{FlashAttention}
\label{alg:flash_attention}
\begin{algorithmic}[1]
\Require $\vect{Q}\in \mathbb{R}^{G\times D_k}$, $\vect{K} \in \mathbb{R}^{S_2\times D_k}$, $\vect{V} \in\mathbb{R}^{S_2\times D_v}$
\Ensure $\vect{O}\in\mathbb{R}^{G\times D_v}$
\State Initialize: $\vect{O} = \mathbf{0}_{G \times D_v}, 
m_0 = \mathbf{-\infty}, \ell_0 = \mathbf{0}_{G \times 1}$
\State Partition $\vect{K},\vect{V}$ into $N$ blocks
\For{$i = 1$ to $N$}
    \State $\vect{S}_i \gets \vect{Q}\vect{K}_i^T$ \Comment{$\Cone$: Compute in Cube cores}
    \State $\begin{aligned}[t]
            m_i &\gets \max(m_{i-1}, \text{rowmax}(\vect{S}_i/\sqrt{D_k})) \\
            \vect{P}_i &\gets \exp(\vect{S}_i/\sqrt{D_k} - m_i) \\
            \ell_i &\gets \ell_{i-1} \cdot \exp(m_{i-1}-m_i) + \text{rowsum}(\vect{P}_i)
           \end{aligned}$ \Comment{$\Vone$: Compute in Vector cores}
    \State $\vect{T}_i \gets \vect{P}_i\vect{V}_i$ \Comment{$\Ctwo$: Compute in Cube cores}
    \State $\vect{O}_i \gets \vect{O}_{i-1} \cdot \exp(m_{i-1}-m_i) + \vect{T}_i$ \Comment{$\Vtwo$: Compute in Vector cores}
\EndFor
\State $\vect{O} \gets \vect{O}_N / \ell_N$
\State \Return $\vect{O}$
\end{algorithmic}
\end{algorithm}

In the decode phase, inputs are $\vect{Q}\in \mathbb{R}^{G\times D_k}$, $\vect{K} \in \mathbb{R}^{S_2\times D_k}$, $\vect{V} \in\mathbb{R}^{S_2\times D_v}$, with output $\vect{O} \in \mathbb{R}^{G\times D_v}$, where typical dimensions are $G=128$, $D_k = 576$, $D_v=512$.

As shown in \Cref{alg:flash_attention}, the kernel is divided into four stages: $\Cone, \Vone, \Ctwo, \Vtwo$, where $\C{i}$ runs on Cube cores and $\V{i}$ on Vector cores. 

Compared to GQA, the $\Vtwo$ stage in MLA incurs significant data movement between GM and UB due to the large intermediate output matrix $\vect{O}_i$ ($D_v=512$ vs. $128$ in GQA). Since $\vect{O}_i$ is typically stored in FP32, its memory footprint is:

$$
G \times D_v \times 4~\text{Bytes} = 256~\text{KB}
$$

Due to the 1:2 ratio between Cube and Vector cores, the memory budget of residing $\vect{O}_i$ per Vector core will be 128 KB -- already insufficient, as the 192 KB UB space must be shared with other intermediate operands. Attempting to keep $\vect{O}_i$ resident in UB thus introduces severe scheduling complexity and resource contention. With MTP, its memory footprint doubles, making UB residency completely infeasible. Consequently, each update in
\begin{equation}
\text{Update($\vect{O}$):} \quad \vect{O}_{i} \gets \vect{O}_{i-1} \cdot \exp(m_{i-1} - m_i) + \vect{P}_i \vect{V}_i. \label{eqn:O_update}
\end{equation}
requires:

\begin{enumerate}
    \item Loading $\vect{O}_{i-1}, \vect{T}_i = \vect{P}_i \vect{V}_i$ from GM to UB, and computing \cref{eqn:O_update};
    \item Writing the updated $\vect{O}_i$ back from UB to GM.
\end{enumerate}

This double memory movement introduces latency and pipeline bubbles, making $\Vtwo$ the critical bottleneck for MLA on Ascend NPUs.

\subsection*{Naive Optimization and Its Pitfall}
To eliminate data movement, one might consider transforming \cref{eqn:O_update} into:
\begin{equation}
\widehat{\vect{O}}_i \gets \widehat{\vect{O}}_{i-1} + \exp(m_i)\vect{P}_i \vect{V}_i, \quad \text{where} \quad \widehat{\vect{O}}_i = \exp(m_i)\vect{O}_i. \label{eqn:inverseupdate}
\end{equation}

This allows direct in-GM updates via \textbf{AtomicAdd}. However, $\exp(m_i)$ may exceed the representable range of FP32, leading to overflow. 
For example, when $m_i>88$, $\exp(m_i)$ overflows in FP32. Actually, \cref{eqn:inverseupdate} is the naive softmax without safe processing, and this is why safe softmax is widely used \cite{onlinesoftmax}. Thus, this naive approach fails numerically.

\subsection{AMLA: Stable In-Memory Updates}

We propose a numerically stable transformation:

\begin{equation}
\vect{\tilde{O}}_i \gets \vect{\tilde{O}}_{i-1} \cdot 2^{n_i-n_{i-1}} + \frac{1}{r_i}\vect{P}_i \vect{V}_i, \label{eqn:amla}
\end{equation}

where

\begin{align*}
    \vect{\tilde{O}}_i &= \vect{O}_i / r_i, \quad
    n_i = \text{round}(-m_i / \ln 2), \quad
    r_i = \exp(-n_i \ln 2 - m_i).
\end{align*}

Here, $1/\sqrt{2} \leq r_i \leq \sqrt{2}$, ensuring numerical stability. The term $\frac{1}{r_i}\vect{P}_i \vect{V}_i$ can be fused into the $\Vone$ stage without extra memory access.

The key innovation lies in computing $\vect{\tilde{O}}_{i-1} \cdot 2^{n_i-n_{i-1}}$ \textit{in-place} in GM using \textbf{AtomicAdd}, leveraging the IEEE 754 floating-point representation.

\subsubsection*{Multiplication via Integer Addition}
Let's first review the FP32 format.
The IEEE 754 single-precision (FP32) format encodes a 32-bit value as three fields~\cite{ieee754}:  
a sign bit $S$ (bit 31),  
8 exponent bits $E$ (bits 23–30),  
and 23 mantissa bits $M$ (bits 0–22).  
The corresponding floating-point value, denoted as $F$, is given by:

\begin{equation}\label{eqn:floatpointing}
F = (-1)^S \times \left(1 + \frac{M}{2^{23}}\right) \times 2^{E-127}, \quad S\in\{0,1\}, \; 0<E<255, \; 0\leq M<2^{23}.
\end{equation}

We neglect subnormal values and do not consider INF/NAN hereafter, as they are irrelevant in typical LLM attention computations.

Simultaneously, interpreting the same 32-bit pattern as a signed integer (INT32), denoted as $I$, yields:

\begin{equation}\label{eqn:integer}
I = -2^{31}S + 2^{23}E + M.
\end{equation}

\begin{figure}
    \centering
    \includegraphics[width=0.8\linewidth]{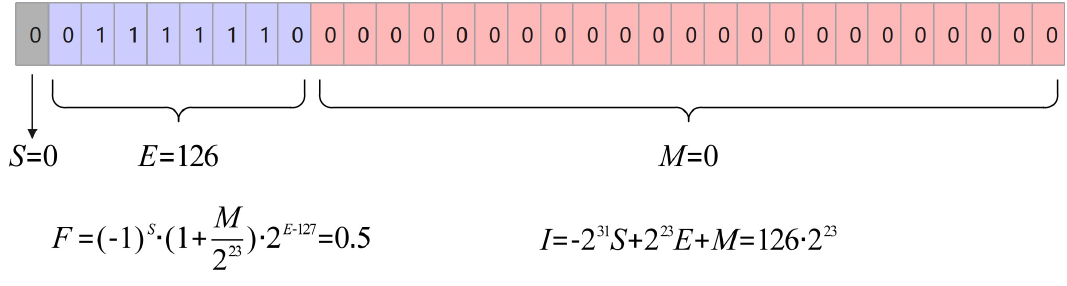}
    \caption{The bit pattern 00111111000000000000000000000000 is 0.5 when interpreted as an FP32 value, and $126 \times 2^{23}$ when interpreted as an INT32 number.}
    \label{fig:fp32int32}
\end{figure}

\begin{example}
The 32-bit pattern 00111111000000000000000000000000 represents the value 0.5 when interpreted as an FP32 number, and represents $126\times 2^{23}$ when interpreted as an INT32 number, see \cref{fig:fp32int32}
\end{example}

This reinterpretation enables a critical insight: multiplying a normalized FP32 value by a power of two ($2^k$) corresponds to adding $k$ to its exponent field $E$, which is a simple fixed-offset addition when reinterpreted as an integer. This allows us to perform exponential rescaling \textit{in-place} in GM via \textbf{AtomicAdd}, without loading the full tensor into compute units.

Defining \textbf{AS\_INT32} and \textbf{AS\_FP32} as bit-preserving reinterpretations, we have
\begin{equation}
\text{AS\_INT32}(F) = I, \quad \text{AS\_FP32}(I) = F.
\end{equation}
Given the above definition, we establish the following lemma:

\begin{lemma}\label{lem:AMLA}
    Given an FP32 number $F$, let $I=AS\_INT32(F)$ denote the integer represented by its binary pattern. Suppose $0 < E < 255$ is the unsigned integer value represented by $F$'s exponent bits. Then, for any integer $n \in \mathbb{Z}$ satisfying $-E < n < 255 -E$, the result of the multiplication $F\times 2^n$ shares the same bit pattern as the result of the integer addition $I+n\times 2^{23}$.
\end{lemma}

\begin{proof}
    Let $S$, $E$, and $M$ denote the numerical values represented by the sign bit (bit 31), exponent bits (bits 23–30), and mantissa bits (bits 0–22) of $F$, respectively. Then $F$ is expressed as \cref{eqn:floatpointing}. Let $I$ in \cref{eqn:integer} be the integer corresponding to $F$'s binary representation.

    Since $0< E + n < 255$ remains within the representable range of UINT8,

    \begin{equation*}
        F' = F\times 2^n = (-1)^S \times (1 + \frac{M}{2^{23}}) \times 2^{(E+n)-127},
    \end{equation*}

    and the INT32 integer corresponding to its binary representation is
    
    \begin{equation*}
        I' = - 2^{31} \times S + 2^{23} \times (E + n) + M = I + n \times 2^{23}.
    \end{equation*}

    Evidently, $F' = F \times 2^n$ and $I' = I + n \times 2 ^{23}$ share identical underlying binary representations. Therefore, $F \times 2^n$ can be implemented via a single integer addition $I + n \times 2 ^{23}$, that is

    \begin{equation} 
        F \times 2^n = \text{AS\_FP32}(\text{AS\_INT32}(F) + n\times 2^{23}).
    \end{equation}
    
\end{proof}

\begin{remark}
This paper presents results for FP32, as it remains the most prevalent precision in large language model training; however, the algorithm is also applicable to other data types like TF32, BF16, and even FP64. 
\end{remark}

\begin{algorithm}
\caption{Ascend MLA (AMLA)}
\label{alg:ascend_mla}
\begin{algorithmic}[1]
\linespread{1.05}\selectfont 
\Require $\vect{Q}\in \mathbb{R}^{G \times D_k}$, $\vect{K}\in \mathbb{R}^{S_2 \times D_k}$, $\vect{V}\in \mathbb{R}^{S_2 \times D_v}$
\Ensure $\vect{O}\in \mathbb{R}^{G \times D_v}$
\State Initialize: $\vect{O} = \mathbf{0}_{G \times D_v}, m_0 = -\infty, \ell_0 = \mathbf{0}_{G \times 1}, n_0 = -\infty, c_0 = 1$
\State Partition $\vect{K},\vect{V}$ into $N$ blocks
\For{$i = 1$ to $N$}
    \State $\vect{S}_i \gets \vect{Q} \vect{K}_i^T$ 
    \State $\vect{S}_i \gets \vect{S}_i / \sqrt{D_k}$, $m_i \gets \max(m_{i-1}, \text{rowmax}(\vect{S}_i))$, $M_\mathrm{up} \gets \exp(m_{i-1} - m_i)$
    \State $n_i \gets \text{round}(-m_i / \ln 2)$, $\vect{P}_i \gets \exp(\vect{S}_i - m_i)$, $\ell_i \gets \ell_{i-1} \cdot M_{up} + \text{rowsum}(\vect{P}_i)$ 
    \State $S_{32} \gets \exp(\ln 2 \cdot (n_i + m_i / \ln 2))$
    \State $S_{16} \gets \text{Cast\_to\_BF16}(S_{32})$, then back to FP32
    \State $c_i \gets S_{32} / S_{16}$, $\varepsilon \gets 1.5 \cdot (c_i / c_{i-1} - 1)$ 
    \State $\vect{P}_i \gets \vect{P}_i \cdot S_{16}$, then cast to BF16
    \State $N \gets \max(n_i - n_{i-1}, -30) + \varepsilon + 10^{-6}$
    \State $N \gets \text{Cast\_to\_int32}(N \cdot 2^{23})$
    \If{$i > 1$}
        \State $\vect{O} \gets \text{AS\_INT32}(\vect{O}) + N$ \Comment{\textbf{$\text{AtomicAdd}\langle\text{INT32}\rangle$ in GM}}
        \State $\vect{O} \gets \text{AS\_FP32}(\vect{O})$
    \EndIf
    \State $\vect{O}_i \gets \vect{P}_i\vect{V}_i$ 
    \State $\vect{O} \gets \vect{O} + \vect{O}_i$ \Comment{\textbf{$\text{AtomicAdd}\langle\text{FP32}\rangle$ in GM}}
\EndFor
\State $\vect{O} \gets \vect{O} / (\ell_N \cdot S_{16})$
\State \Return $\vect{O}$
\end{algorithmic}
\end{algorithm}

\begin{remark}
    In \cref{alg:ascend_mla}, the row scaling operation on Line 10 can be directly integrated into Line 7 and implemented via vector subtraction to reduce a portion of the computational load.
\end{remark}

\subsubsection*{In-Memory Update with AtomicAdd}

Applying \Cref{lem:AMLA}, we compute:

\begin{equation}\label{amla:update1}
\vect{\tilde{O}}_{i-1} \cdot 2^{n_i-n_{i-1}} = \text{AS\_FP32}\left(\text{AS\_INT32}(\vect{\tilde{O}}_{i-1}) + (n_i - n_{i-1}) \cdot 2^{23}\right).
\end{equation}

This enables in-GM updates via \textbf{AtomicAdd}. Combined with accumulating $\frac{1}{r_i}\vect{P}_i \vect{V}_i$, the entire $\vect{\tilde{O}}_i$ update occurs in memory --- eliminating transfers between GM and UB.

\subsubsection*{Error Compensation}
A subtle precision issue arises due to BF16 quantization in intermediate steps; we address it via error compensation (see \Cref{app:errorcompensation}).

\subsection{Comparison between Base and AMLA}
\label{subsec:computestage}

We recall \textbf{Base} as the standard FlashAttention-style implementation (see \Cref{alg:flash_attention}), which decomposes attention into four distinct pipeline stages: $\Cone$ ($\vect{Q}\vect{K}^T$ on Cube cores), $\Vone$ (online softmax on Vector cores), $\Ctwo$ ($\vect{P}\vect{V}$ on Cube cores), and $\Vtwo$ (rescaling on Vector cores). In MLA, the $\Vtwo$ stage becomes a severe bottleneck due to the large FP32 output tensor $\vect{O}_i$, forcing repeated transfers between GM and UB that stall the pipeline and underutilize compute units.

AMLA eliminates this bottleneck by collapsing $\Vtwo$ via numerically stable in-memory computations: Rescaling is converted to integer addition (via \textbf{$\text{AtomicAdd}\langle\text{INT32}\rangle$}), and output accumulation occurs directly in GM (via \textbf{$\text{AtomicAdd}\langle\text{FP32}\rangle$}). FlashAttention-2 conducts the normalization in the final Vector stage~\cite{FlashAttention-2}, which is also adopted by AMLA (Line 20 in \cref{alg:ascend_mla}). Apart from the final one, the pipeline is reduced to three stages ($\Cone, \Vone, \Ctwo$) without sacrificing numerical accuracy (validated in \Cref{sec:accuracy}), while the repeated data movement in $\Vtwo$ is eliminated at the same time. The comparison between Base and AMLA in the rescaling phase is shown in \cref{fig:amla_base}.

\begin{figure}[htbp]
    \centering
    \includegraphics[width=0.99\linewidth]{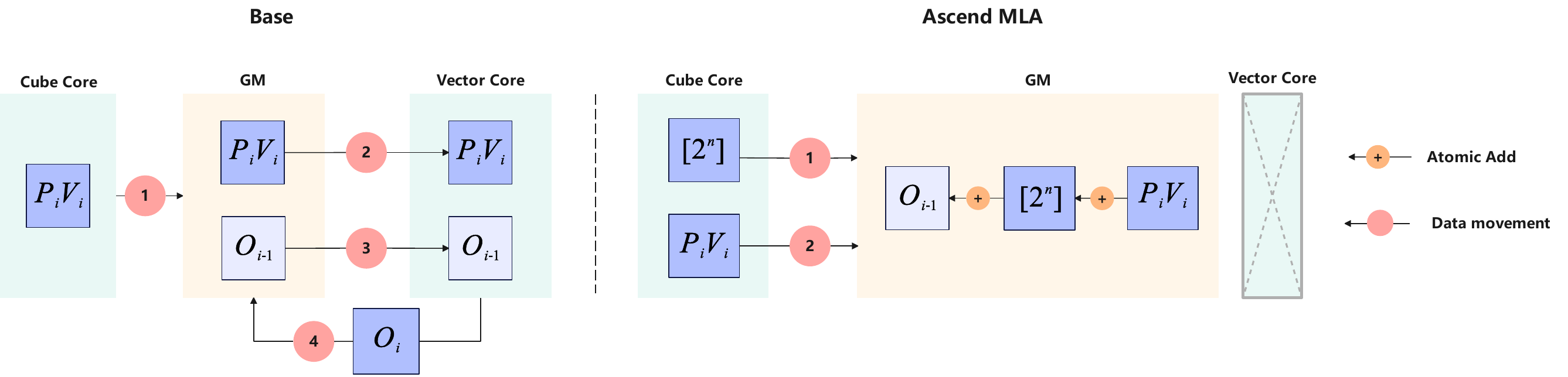}
    \caption{Base vs. Ascend MLA.}
    \label{fig:amla_base}
\end{figure}
\section{Implementation}
\label{sec:implementation}

This section presents hardware-aware optimizations tailored for Ascend NPUs. We achieve two levels of pipelining: (1) \textit{inter-stage} overlap between Cube and Vector cores across $\Cone$, $\Vone$, $\Ctwo$; and (2) \textit{intra-stage} overlap of computation and memory movement within each Cube stage. As introduced in \Cref{sec:Algorithm}, AMLA consists of three core stages: $\Cone$ and $\Ctwo$ (matrix multiplication on Cube cores), and $\Vone$ (vector operations on Vector cores).

\begin{itemize}
    \item \cref{subsec:preloadpipeline} introduces the \textbf{Preload Pipeline}, which decouples inter-stage dependencies, enabling full overlap of Vector stages by concurrent Cube stages. This renders the kernel \textbf{Cube-bound} --- its performance is ultimately limited by Cube core latency.
    \item As established in \Cref{subsec:arithmeticintensity}, MLA is inherently compute-bound. Thus, within Cube stages, we further optimize for \textbf{computation-memory overlap}. \cref{subsec:cubestagepipeline} details our hierarchical tiling and pipeline design, derived from theoretical bandwidth-compute balance analysis, to maximize Flops utilization.
\end{itemize}

\subsection{Preload Pipeline}
\label{subsec:preloadpipeline}

This subsection introduces the Preload Pipeline, a novel two-phase architecture designed to eliminate inter-stage dependencies and achieve full hardware utilization. Our presentation is structured as follows: first, we outline the core architectural design, which separates an initial dependency-resolving \textit{Preload} phase from a stall-free \textit{Steady Pipeline Loop}. Second, we formalize its efficiency by defining the \textit{Preload count}, a metric for warm-up overhead, and derive a tight theoretical bound for its minimization. Finally, we demonstrate the pipeline's practical application to the AMLA kernel. To ground this discussion, we begin with the generalized dependency chain common in such workloads, $\Cone \to \Vone \to \Ctwo \to \Vtwo$ (where AMLA is a special case with $\Vtwo = 0$), illustrating how naive execution leads to underutilization.

\subsubsection{Two-phase Pipeline Overview}

We first propose a two-phase pipeline architecture: \textit{Preload} and \textit{Steady Pipeline Loop} (see \cref{fig:preload}).

\begin{enumerate}
    \item \textbf{Preload:} Resolves upstream dependencies before the main loop begins, enabling flexible scheduling of the first cycle (Cycle 1 in \cref{fig:preload}). For instance, if $\Ctwo$ is scheduled before $\Cone$ (despite $\Ctwo$ depending on $\Vone$), the Preload phase pre-executes $\Cone$ and $\Vone$ to satisfy $\Ctwo$'s dependencies ($\Cone \to \Vone \to \Ctwo$).
    \item \textbf{Steady Pipeline Loop:} Repeated execution of identical \textit{Cycles}, each encapsulating $\Cone$, $\Vone$, $\Ctwo$, $\Vtwo$. Thanks to dependency decoupling in Preload, operations within a Cycle can be reordered freely without stalls.
\end{enumerate}

\begin{figure}[h!]
    \centering
    \includegraphics[width=\textwidth]{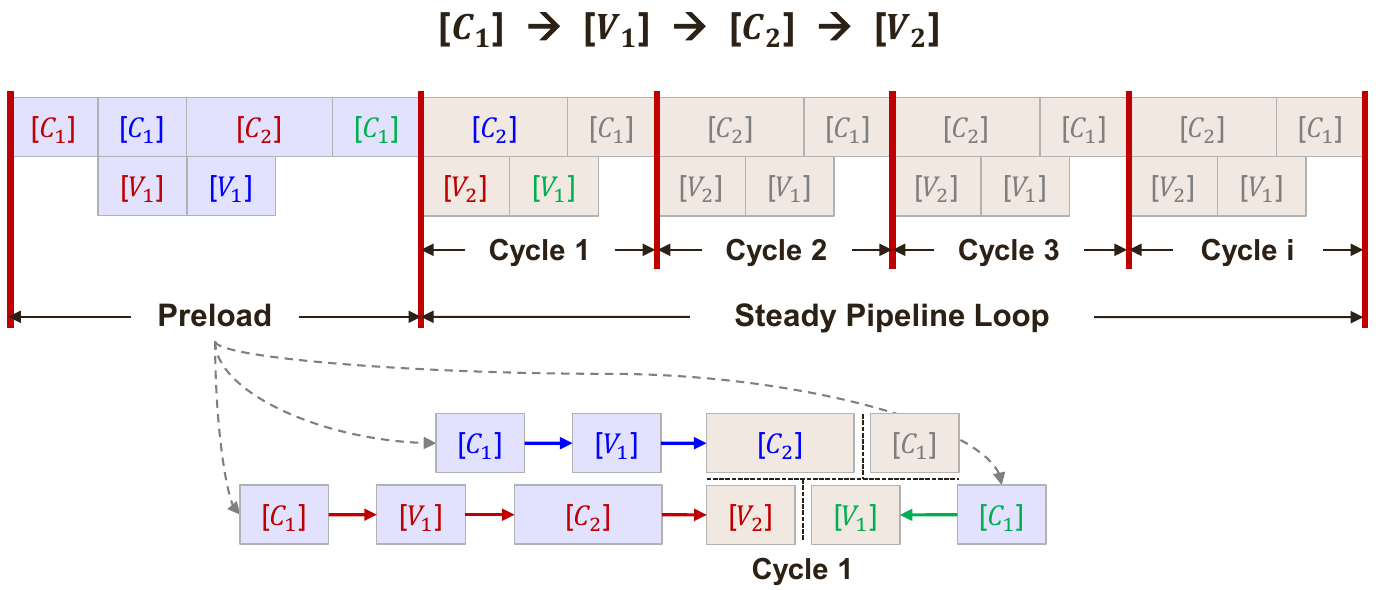}
    \caption{Two-phase pipeline: (1) Preload resolves initial dependencies; (2) Steady Loop executes Cycles with maximal concurrency.}
    \label{fig:preload}
\end{figure}

This architecture generalizes to arbitrary $n$ CV-pair chains and imposes no constraints on individual stage durations or aggregate Cube/Vector timing, ensuring optimal performance regardless of whether the workload is Cube-bound or Vector-bound.

\subsubsection{Key Conclusions}
\label{sec:preload-conclusions}

We define the \textit{Preload count} as the minimum number of $\Cone$ stages executed in the Preload phase. This metric is equivalent to the number of operation blocks per Cycle that require external dependency resolution. For example, the Preload number sin \Cref{fig:preload} is 3, since there are 3 $\Cone$ stages in Preload. {As illustrated in \cref{fig:preload2}, the Preload count can be minimized by reordering intra-Cycle blocks, which directly reduces the pipeline warm-up duration.

\begin{figure}[h!]
    \centering
    \includegraphics[width=\textwidth]{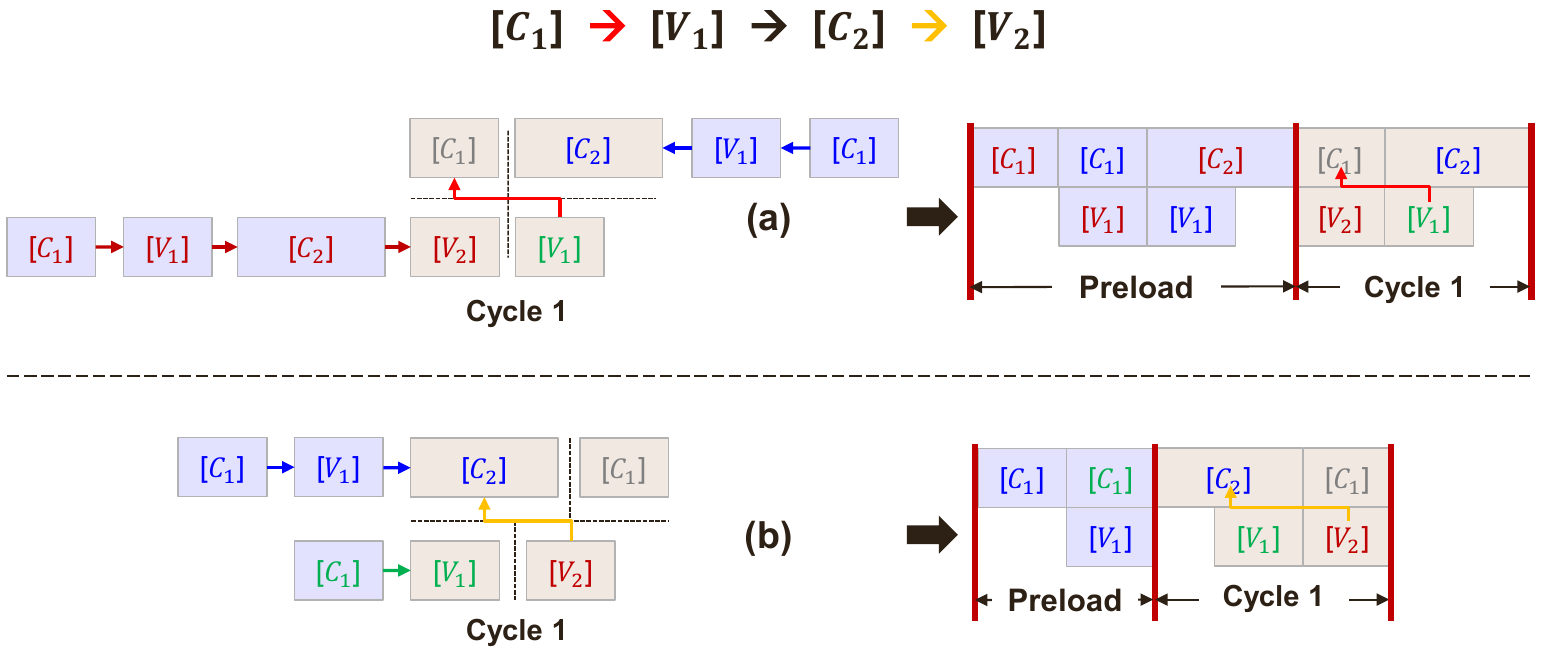}
    \caption{Rearrange Cycles to shorten the Preload phase. (a) Decoupling an earlier dependency $\Cone\to\Vone$ reduces the Preload phase by one operation block ($\Cone$) compared with \cref{fig:preload}. (b) Decoupling a latter dependency $\Ctwo\to\Vtwo$ eliminates three operation blocks ($\Cone, \Vone, \Ctwo$) from the Preload phase compared with (a), highlighting that rearranging the internal operation blocks of a Cycle achieves a more significant reduction in Preload.}
    \label{fig:preload2}
\end{figure}

Combining the upper bound (\cref{lem:preload_upper}) and existence guarantee (\cref{thm:existence_pipeline}) from Appendix \ref{app:preload-proofs}, we derive a tight theoretical guarantee:

\begin{theorem}\label{thm:preload_tight_bound}
For any $n$-stage $\C{} \V{}$ chain $\Cone \to \Vone \to \cdots \to \Cn \to \Vn$, where the latency of each $\C{i}$ and $\V{i}$ can be arbitrary:
\begin{enumerate}
    \item A pipeline with Preload count $< n$ is not always achievable.
    \item A pipeline with Preload count $= n$ always exists, with all intra-Cycle dependencies of the form $\V{} \to \C{}$.
\end{enumerate}
\end{theorem}

This bound ensures minimal pipeline warm-up and maximal hardware utilization: the Preload phase resolves initial dependencies efficiently, while the Steady Loop sustains stall-free execution, eliminating idle cycles and achieving near-peak Cube core utilization.

\subsubsection{Preload Pipeline for AMLA}

The above principles of the Preload pipeline are directly applied in the AMLA kernel, as illustrated in \Cref{fig:amla_pipeline}. AMLA can be viewed as a specific instance of the above general model where the number of CV pairs $n=2$ (with stages $\Cone \to \Vone \to \Ctwo$) and the final vector stage $\Vtwo$ has a duration of zero. In accordance with \Cref{thm:preload_tight_bound}, we adopt a Preload count of 2: in the Preload phase, we first execute $\Cone$ and $\Vone$ to decouple $\Ctwo$ of the first Cycle, and then execute $\Cone$ to decouple $\Vone$ of the first Cycle. After entering the Steady Pipeline Loop, each Cycle encapsulates $[\Cone, \Ctwo, \Vone]$ with no intra-Cycle dependencies. At the same time, in Steady Pipeline Loop, the Vector stage is fully overlapped by the ongoing Cube stage at the outer level, and the kernel is Cube-bound. Since the Preload phase issues $[\Cone, \Cone, \Vone]$, we drain the pipeline at the tail by appending the remaining downstream stages $[\Ctwo, \Ctwo, \Vone]$ and a $\text{Last} \ [V]$, thereby closing the dependency chain and completing the operator.

\begin{figure}[h!]
    \centering
    \includegraphics[width=\textwidth]{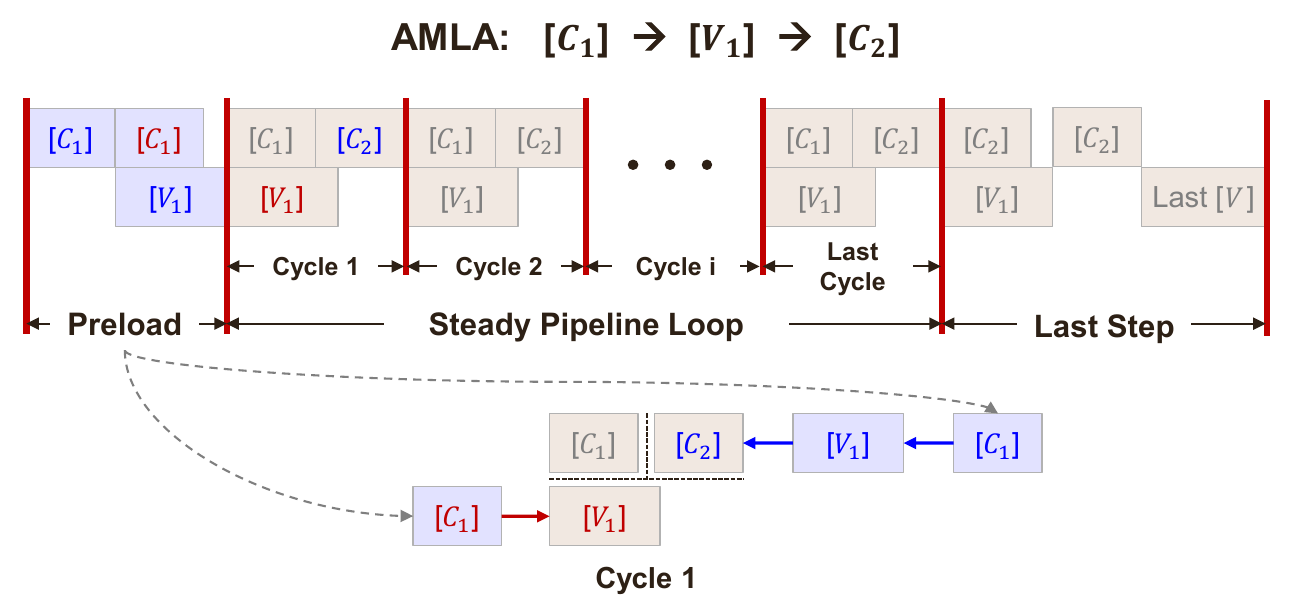}
    \caption{Preload pipeline architecture for AMLA}
    \label{fig:amla_pipeline}
\end{figure}

\subsection{Hierarchical Tiling}
\label{subsec:cubestagepipeline}

In this subsection, we detail our optimized tiling and pipeline strategy. We categorize the Cube stages into four distinct pipelines:
\begin{itemize}
    \item \MMAD. Multiply–accumulate operations.
    \item \MTEone. Data movement from L1 cache to L0A/L0B cache.
    \item \MTEtwo. Data movement from GM to L1 cache.
    \item \FixP. Data movement from L0C cache to GM.
\end{itemize}

Both $\Cone$ and $\Ctwo$ follow the pipeline order:
\begin{equation}
    \MTEtwo \to \MTEone \to \MMAD \to \FixP.
\end{equation}

We design a \textbf{hierarchical tiling strategy} that balances \MTEtwo/\MTEone/\FixP\ bandwidth and Cube core compute throughput. Let $M, N, K$ denote the tensor dimensions per FlashAttention iteration. We fix the KV block size to 512, yielding:
\begin{itemize}
    \item $\Cone$: $N = 512, K = 576$
    \item $\Ctwo$: $N = K = 512$
\end{itemize}

Data is tiled hierarchically:
\begin{itemize}
    \item From GM to L1: tiled as $\singleM \times \singleK$, $\singleN \times \singleK$
    \item From L1 to L0A/B: further tiled as $\baseM \times \baseK$, $\baseN \times \baseK$ (see \cref{fig:hier-tiling})
\end{itemize}

\begin{figure}[h!]
    \centering
    \includegraphics[width=\textwidth]{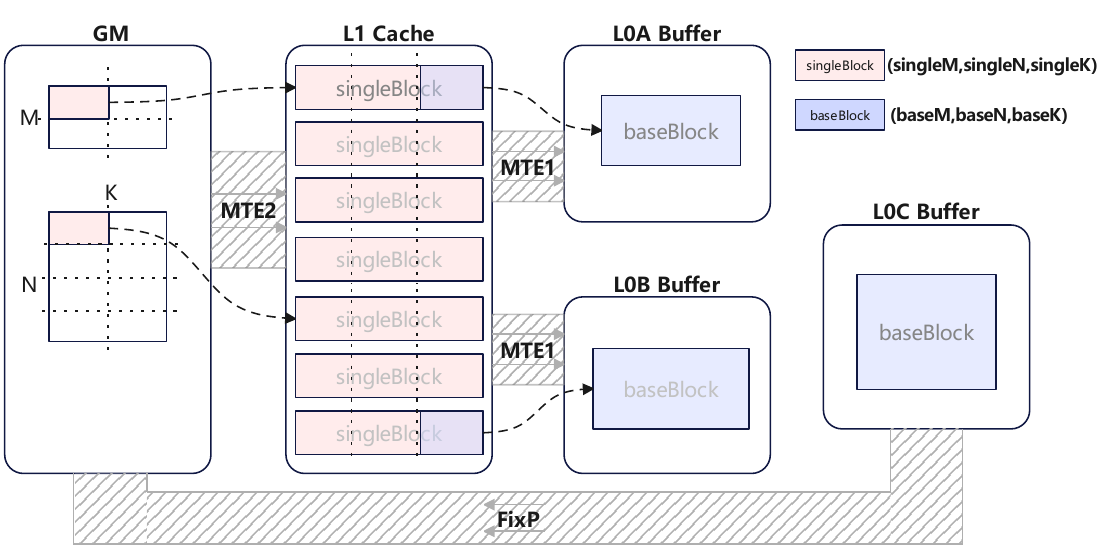}
    \caption{Hierarchical tiling across memory hierarchy.}
    \label{fig:hier-tiling}
\end{figure}

\paragraph{FlashAttention Block Size.}
To overlap data transfer with computation, we require:
\begin{equation*}
    \frac{M \cdot N \cdot K \cdot 2}{n_{\text{c}} \cdot f \cdot F_{\text{BF16}}} \ge \frac{N \cdot K \cdot \text{sizeof}(\text{BF16})}{\text{BW}_\text{HBM}}.
\end{equation*}
In the case of Ascend 910, the number of Cube cores is $n_{\text{c}}=48$, HBM bandwidth is $\text{BW}_\text{HBM}=3.2\text{ TB/s}$, $f$ is the frequency, and $F_{\text{BF16}}$ is the peak BF16 FLOPS for each Cube core per clock cycle.

We exclude the \MTEtwo\ overhead of $\vect{Q}$ and $\vect{P}$ because: (1) they reside in L1 within an iteration; (2) after initial load, subsequent accesses are served from L2 Cache. For Ascend 910, we set $M=256$ (DeepSeek-V3, 128 query heads) to balance compute and bandwidth.

\paragraph{L1 Cache Tiling.}
We retain $\vect{Q}$/$\vect{P}$ in L1 within each FlashAttention loop to minimize redundant KVCache transfers. We reserve 288 KB for $\vect{Q}$ (also used for $\vect{P}$). To avoid the long warm-up overhead of memory access, we pipeline transfers in 128 × 288 ($\vect{Q}$) or 128 × 256 ($\vect{P}$) stripes.

The remaining 224 KB L1 space holds $\vect{K}$ (512×576) and $\vect{V}$ (512×512) via a 3-buffer pipeline (72 KB/buffer):
\begin{itemize}
    \item $\Cone$: $\singleM = 128$, $\singleK = 288$, $\singleN = 256$
    \item $\Ctwo$: $\singleM = 128$, $\singleK = 256$, $\singleN = 256$
\end{itemize}

The 512 KB L1 cache is partitioned into seven 72 KB buffers: four for $\vect{Q}$/$\vect{P}$, three for $\vect{K}$/$\vect{V}$ (see \cref{fig:hier-tiling}).

\begin{remark}
    $\Cone$ and $\Ctwo$ share identical L1 Cache tiling to eliminate pipeline bubbles between stages.
\end{remark}

\paragraph{L0 Cache Tiling.}
With double buffering and L0 sizes (L0A/B: 64 KB, L0C: 128 KB), we enforce:
\begin{equation*}
\begin{aligned}
    &\baseM \cdot \baseK \cdot \text{sizeof(BF16)} \le 32\,\text{KB}, \\
    &\baseN \cdot \baseK \cdot \text{sizeof(BF16)} \le 32\,\text{KB}, \\
    &\baseM \cdot \baseN \cdot \text{sizeof(FP32)} \le 64\,\text{KB}.
\end{aligned}
\end{equation*}

To fully exploit L0 cache, we set $\baseM = \baseN = 128$, and adjust $\baseK$ per stage:
\begin{itemize}
    \item $\Cone$: $\baseK = 96$ (to match 576 input dim)
    \item $\Ctwo$: $\baseK = 128$ (to match 512 input dim)
\end{itemize}
This ensures balanced compute load per \MMAD.

\paragraph{\FixP.}
To minimize writeback traffic, we accumulate multiple results in L0C before bulk transfer to GM.

\medskip

The resulting pipeline (see \cref{fig:7buf}) achieves perfect overlap of memory and compute, ensuring the AMLA kernel remains strictly compute-bound.

\begin{figure}[h!]
    \centering
    \includegraphics[width=\textwidth]{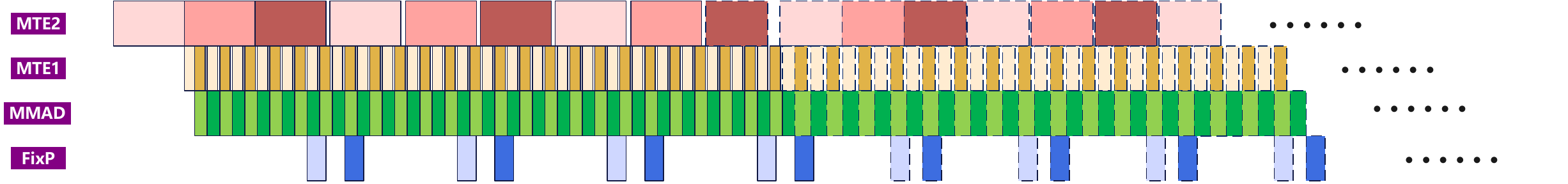}
    \caption{Optimized Cube stage pipeline. Colors denote buffer partitions (triple-buffer L1 for $\vect{K}$ and $\vect{V}$, double-buffer L0). Solid/dashed lines: $\Cone$/$\Ctwo$ stages. The MTE2 of $\vect{Q}$ and $\vect{P}$ is not included.}
    \label{fig:7buf}
\end{figure}
\section{Numerical Results}

In this section, we evaluate the accuracy and performance of AMLA through comprehensive experiments. 

\subsection{Accuracy} \label{sec:accuracy}

We evaluate AMLA's accuracy by comparing it against two baselines: \textbf{Golden} and \textbf{Base}:
\begin{itemize}
    \item \textbf{Golden}: High-precision FP32 attention computation implemented on CPU, serving as the ground-truth reference.
    \item \textbf{Base}: CPU simulation of the standard FlashAttention implementation using mixed-precision matrix multiplications, with final outputs in FP16 or BF16.
    \item \textbf{AMLA}: NPU implementation based on \cref{alg:ascend_mla}.
\end{itemize}

To quantify accuracy, we compute the relative Frobenius-norm error between each method's output and the Golden reference:
$$
\mathcal{E}(A,B) = \frac{\|A-B\|_F}{\|B\|_F+ \varepsilon},
$$
where $\|C\|_F = \sqrt{\operatorname{Tr}(CC^H)} = \sqrt{\sum_{i,j} |c_{ij}|^2}$, and $\varepsilon = 10^{-10}$ for numerical stability.

We test across multiple input distributions for $\vect{Q}, \vect{K}, \vect{V}$ in BF16: Gaussian $\mathcal{N}(0,\sigma^2)$ with varying $\sigma$, and uniform $\mathcal{U}(a,b)$ with different ranges. For each distribution, we generate 100 random samples under typical settings (context length = 8K) and report average errors. As shown in \cref{tab:gaussianexp} and \cref{tab:uniformexp}, AMLA achieves accuracy nearly identical to Base.

\begin{table}[htbp]
\centering
\begin{tabular}{ccccccc}
\toprule
$\mathcal{E}(\cdot,\text{Golden})$ & $\mathcal{N}(0,1)$ & $\mathcal{N}(0,4)$ & $\mathcal{N}(0,9)$ & $\mathcal{N}(0,16)$ & $\mathcal{N}(0,25)$ & $\mathcal{N}(0,100)$ \\
\midrule
Base & 1.77E-03 & 1.74E-03 & 1.65E-03 & 1.51E-03 & 1.33E-03 & 7.82E-04 \\
AMLA & 1.81E-03 & 1.75E-03 & 1.66E-03 & 1.51E-03 & 1.35E-03 & 7.86E-04 \\
\bottomrule
\end{tabular}
\vspace{0.2cm}
\caption{Accuracy comparison under Gaussian input distributions.}
\label{tab:gaussianexp}
\end{table}

\begin{table}[htbp]
\centering
\begin{tabular}{ccccccc}
\toprule
$\mathcal{E}(\cdot,\text{Golden})$ & $\mathcal{U}(-1,1)$ & $\mathcal{U}(-3,3)$ & $\mathcal{U}(-5,5)$ & $\mathcal{U}(-10,10)$ & $\mathcal{U}(-20,20)$ & $\mathcal{U}(-60,60)$ \\
\midrule
Base & 1.97E-03 & 1.77E-03 & 1.69E-03 & 1.24E-03 & 7.04E-04 & 2.26E-04\\
AMLA & 2.01E-03 & 1.78E-03 & 1.69E-03 & 1.24E-03 & 7.04E-04 & 2.26E-04\\
\bottomrule
\end{tabular}
\vspace{0.2cm}
\caption{Accuracy comparison under uniform input distributions.}
\label{tab:uniformexp}
\end{table}

\subsection{Performance}

As a core component in LLM inference, the MLA kernel's latency significantly impacts overall system performance, especially in long-context scenarios. With the adoption of DeepEP~\cite{DeepEP} and prefill-decode disaggregation~\cite{HisiCM384}, the computational burden of MoE layers is reduced, making attention kernel efficiency a key bottleneck for throughput.

Leveraging our algorithmic and hardware-aware optimizations, we implement the AMLA kernel on Ascend 910 to maximize hardware utilization.

We benchmark AMLA on Ascend 910 against a GPU with 989 TFLOPS of BF16 compute performance and 3.35 TB/s memory bandwidth in the decode phase. Experiments use typical BF16 configurations. Since tensor parallelism is typically not applied to the MLA component, we set the number of query heads to 128 and key heads to 1. Batch size is fixed at 96. We evaluate both non-MTP ($S_q = 1$) and MTP ($S_q = 2$) scenarios, with variable context lengths $S_k$.

\begin{figure}
    \centering
    \includegraphics[width=1\linewidth]{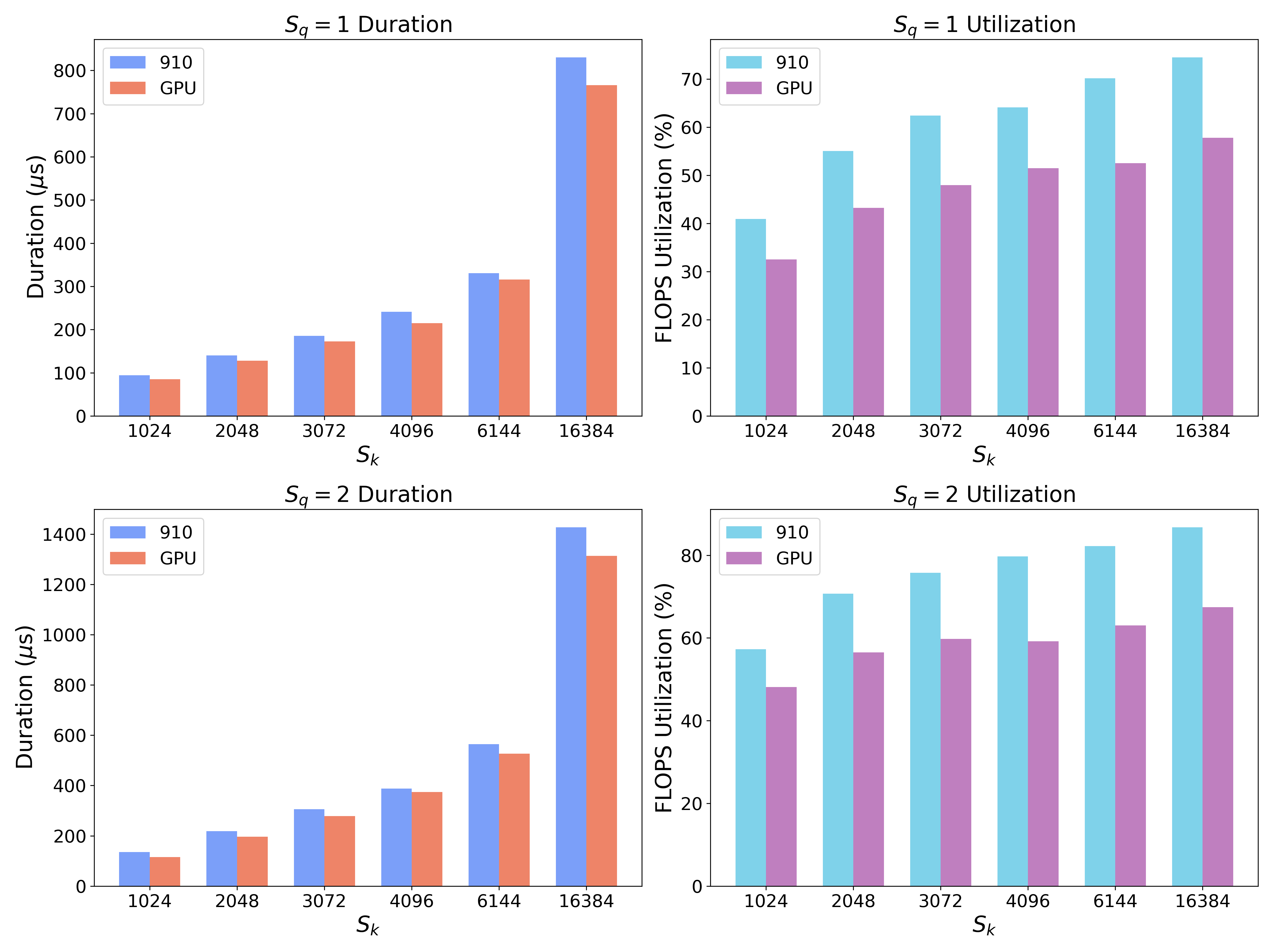}
    \caption{Performance comparison.}
    \label{fig:perf_exp}
\end{figure}

The comparison of kernel duration and FLOPS utilization is shown in \cref{fig:perf_exp}, and the specific results are shown in \cref{tab:perfexp}. Based on the experimental results, AMLA achieves up to \textbf{86.8\%} FLOPS utilization, a testament to our hardware-aware optimizations.

\begin{table}[htbp]
\centering
\begin{subtable}[t]{\textwidth}
\caption{Part 1}
\centering
\begin{tabular}{@{}cccccccccc@{}}
\toprule
  & $S_k$  & \multicolumn{2}{c}{1024} & \multicolumn{2}{c}{2048} & \multicolumn{2}{c}{3072} \\
\cmidrule(lr){3-4} \cmidrule(lr){5-6} \cmidrule(lr){7-8}
$S_q$ & Hardware &duration ($\mu$s) & FU & duration ($\mu$s) & FU & duration ($\mu$s) & FU  \\
\midrule
\multirow{2}{*}{1} & 910  & 95 & 40.9\% & 140 & 55.1\% & 186 & 62.4\% \\
                   & GPU  & 85 & 32.6\% & 128 & 43.3\% & 173 & 48.0\% \\
\midrule
\multirow{2}{*}{2} & 910  & 135 & 57.3\% & 219 & 70.7\% & 306 & 75.8\% \\
                   & GPU  & 115 & 48.1\% & 196 & 56.5\% & 278 & 59.8\% \\
\bottomrule
\end{tabular}
\end{subtable}

\begin{subtable}[t]{\textwidth}
\caption{Part 2}
\centering
\begin{tabular}{@{}cccccccccc@{}}
\toprule
  & $S_k$  & \multicolumn{2}{c}{4096} & \multicolumn{2}{c}{6144} & \multicolumn{2}{c}{16384} \\
\cmidrule(lr){3-4} \cmidrule(lr){5-6} \cmidrule(lr){7-8}
$S_q$ & Hardware &duration ($\mu$s) & FU & duration ($\mu$s) & FU & duration ($\mu$s) & FU  \\
\midrule
\multirow{2}{*}{1} & 910  & 241 & 64.1\% & 331 & 70.2\% & 830 & 74.5\% \\
                   & GPU  & 215 & 51.5\% & 316 & 52.6\% & 766 & 57.8\% \\
\midrule
\multirow{2}{*}{2} & 910  & 388 & 79.7\% & 565 & 82.2\% & 1427 & \textbf{86.8\%} \\
                   & GPU  & 374 & 59.2\% & 527 & 63.0\% & 1314 & 67.4\% \\
\bottomrule
\end{tabular}
\end{subtable}
\vspace{0.2cm}
\caption{Performance comparison. FU denotes FLOPS utilization.}
\label{tab:perfexp}
\end{table}
\section{Summary}

In this work, we present AMLA --- a highly optimized MLA kernel tailored for efficient decoding on Ascend 910 NPUs. Our key contribution lies in a synergistic combination of algorithmic innovation and hardware-aware optimization, enabling unprecedented efficiency in attention computation for long-context LLM inference.

First, by exploiting the binary correspondence between floating-point and integer representations, we reformulate the rescaling of $\vect{O}$ blocks to replace multiplications with additions. This eliminates redundant arithmetic operations and reduces associated memory traffic. Second, we introduce hardware-aware optimizations, including a Preload Pipeline strategy and optimized hierarchical tiling, which fully leverage Ascend NPU’s native support for in-memory computing and fine-grained orchestration of heterogeneous compute units.

Evaluated on Ascend 910, AMLA achieves up to \textbf{86.8\%} FLOPS utilization --- surpassing the state-of-the-art FlashMLA~\cite{FlashMLA} (67.4\% utilization) by a significant margin.

Together with our accuracy validation, this work provides a complete, high-performance, and production-ready solution for deploying MLA-based attention on Ascend NPUs.

\bibliographystyle{plain}
\bibliography{ref.bib} 
\newpage
\appendix

\section{Error Compensation in AMLA} \label{app:errorcompensation}

Although \cref{eqn:amla} is mathematically equivalent to \cref{eqn:O_update}, its numerical behavior under finite-precision floating-point arithmetic differs significantly in practice. While FP32 computation introduces negligible error, especially since the final output is typically cast to BF16, the use of low-precision matrix multiplications (e.g., $\vect{P}_i \vect{V}_i$) on AI accelerators introduces non-negligible rounding errors that accumulate across iterations.

Specifically, the on-chip computation of \cref{eqn:amla} can be expressed as:
\begin{equation}
    \vect{\tilde{O}}_i = \vect{\tilde{O}}_{i-1} \times 2^{n_i - n_{i-1}} + \frac{1}{r_i} \left( \vect{P}_i \right)_{\text{BF16}} \left( \vect{V}_i \right)_{\text{BF16}},
\end{equation}
where $(\cdot)_{\text{BF16}}$ denotes explicit casting to BF16 to reduce both computational cost and memory traffic. Modern specialized hardware units, including Ascend Cube cores, typically accept BF16 inputs and perform internal accumulation in FP32 precision.

The scaling factor $1/r_i$ can be applied in two ways:
\begin{enumerate}
    \item \textbf{Post-multiplication scaling}: Compute $\vect{T} = (\vect{P}_i)_{\text{BF16}} (\vect{V}_i)_{\text{BF16}}$ first, then scale by $1/r_i$. This requires scalar-matrix multiplication in FP32, which has to be executed on the Vector core and incurs significant memory access overhead.
    
    \item \textbf{Pre-multiplication scaling}: Scale $\vect{P}_i$ by $1/r_i$ first, then perform the matrix multiplication $\left( \vect{P}_i / r_i \right)_{\text{BF16}} (\vect{V}_i)_{\text{BF16}}$. This approach, while more straightforward to implement, amplifies rounding error: since $1/r_i$ is cast to BF16, the term $\left(1/r_i\right)_{\text{BF16}}$ introduces quantization error even when $\vect{P}_i$ contains values as benign as 1 (which is exactly representable in BF16). This error propagates multiplicatively and significantly degrades output accuracy.
\end{enumerate}

To mitigate this, we reformulate the computation by defining $1/r'_i = (1/r_i)_{\text{BF16}}$, leading to:
\begin{equation}
    \vect{\tilde{O}}_i = \vect{\tilde{O}}_{i-1} \times 2^{n_i - n_{i-1}} + \frac{r'_i}{r_i} \left( \frac{\vect{P}_i}{r'_i} \right)_{\text{BF16}} \left( \vect{V}_i \right)_{\text{BF16}}.
\end{equation}

Let $c_i = r_i / r'_i$ and $\vect{\bar{O}}_i = \vect{\tilde{O}}_i \cdot c_i$. Substituting yields the compensated recurrence:
\begin{equation}\label{eqn:modifiediteration}
    \vect{\bar{O}}_i = \vect{\bar{O}}_{i-1} \times 2^{n_i - n_{i-1}} \times \frac{c_{i}}{c_{i-1}} + \left( \frac{\vect{P}_i}{r'_i} \right)_{\text{BF16}} \left( \vect{V}_i \right)_{\text{BF16}}.
\end{equation}

This formulation requires scaling $\vect{\bar{O}}_{i-1}$ by both $2^{n_i - n_{i-1}}$ and $c_i / c_{i-1}$. The former is efficiently handled via \textbf{AtomicAdd} (cf. \cref{lem:AMLA}); the latter is non-trivial but tractable due to the proximity of $c_{i} / c_{i-1}$ to 1. Since BF16 has a relative precision of approximately $1/256$, we have $c_{i} / c_{i-1} = 1 + \epsilon$ with $|\epsilon| < 1/256$.

Considering a floating-point number $F$, 
\begin{equation}
F = (-1)^S \left(1 + \frac{M}{2^{23}}\right) 2^{E-127}
\end{equation}
Multiplying $F$ by $1+\epsilon$ yields:
\begin{equation}
    F \cdot (1+\epsilon) = (-1)^S \cdot 2^{E-127} \cdot \left(1 + \frac{M + 2^{23}\epsilon + \epsilon M}{2^{23}}\right).
\end{equation}

To approximate the integer representation of the result, we estimate $M \approx 2^{22}$ (the midpoint of the mantissa range), leading to:
\begin{align*}
    \text{AS\_INT32}(F \cdot (1+\epsilon))
    &\approx \text{AS\_INT32}(F) + \text{round}\left(2^{23} \epsilon + 2^{22} \epsilon\right) \\
    &= \text{AS\_INT32}(F) + \text{round}\left(1.5 \cdot 2^{23} \epsilon\right).
\end{align*}

This additive integer adjustment enables efficient implementation via atomic operations without invoking costly FP32 multiplies. Combined with \cref{amla:update1}, this yields the final \cref{alg:ascend_mla}.

Critically, this compensation introduces only minimal overhead confined to the $\V{1}$ stage and has negligible impact on overall kernel performance while substantially improving numerical fidelity.

\section{Proofs for the Preload Pipeline}
\label{app:preload-proofs}

\subsection{Optimization Analysis of the Preload Pipeline}

As defined in \cref{sec:preload-conclusions}, each internal dependency chain established within a Cycle reduces the required Preload count by one. For instance, in \cref{fig:preload2}a, if $\V{1}$ can directly consume the result of $\C{1}$ computed within the same Cycle (i.e., $\C{1} \to \V{1}$ is internally resolved), then $\C{1}$ need not be precomputed in the Preload phase. This internal dependency decouples $\V{1}$ from its external dependency on $\C{1}$, thereby reducing the Preload burden.

Importantly, later stages in the dependency chain (e.g., $\V{2}$, $\V{3}$) typically require longer prerequisite sequences to be internally satisfied. Hence, resolving dependencies for later blocks, such as $\C{2} \to \V{2}$ in \cref{fig:preload2}b, yields greater Preload reduction (e.g., eliminating 3 blocks instead of 1). This observation suggests that \textbf{maximizing the number of internal dependency chains} is key to minimizing the Preload overhead.

\begin{lemma}
\label{lem:preload-count}
For a stage consisting of $n$ CV pairs, the total length of the sequential dependency chain is $2n - 1$. Let $s$ denote the number of internal dependency chains constructed within a Cycle. Then the Preload count is given by:
\begin{equation*}
    \text{Preload count} = (2n - 1) - s.
\end{equation*}
Each internal chain reduces the need for external precomputation, so a larger $s$ leads to a smaller Preload count.
\end{lemma}

This leads to the central question: Under the constraint of preserving pairwise CV dependencies, does there exist a universal upper bound $\bar{s}$ on the number of internal dependency chains that can be achieved via rearrangement, regardless of individual stage durations? If so, what is $\bar{s}$?

\subsection{Problem Construction and Upper Bound Derivation}

To address this, we consider two cases based on the relative total durations of Cube and Vector stages: $\sum \V{i} \leq \sum \C{i}$ (Cube-dominated) and $\sum \V{i} \geq \sum \C{i}$ (Vector-dominated). We first analyze the Cube-dominated case.

Consider $n$ CV pairs forming the chain $\C{1} \to \V{1} \to \C{2} \to \V{2} \to \cdots \to \C{n} \to \V{n}$. We aim to prove that no pipeline can guarantee more than $\bar{s} = n - 1$ internal dependency chains for arbitrary stage durations.

To establish this upper bound, we construct an adversarial scenario: suppose there exists a Vector stage $\V{k}$ such that
\begin{equation}
\label{eqn:extreme-case}
    \V{k} + \C{j} > \sum_{i=1}^n \C{i}, \quad \forall j \in \{1,2,\dots,n\}.
\end{equation}
This implies that $\V{k}$ cannot overlap with any Cube stage without exceeding the total cube duration. Given that all Vector stages must be overlapped by the cumulative Cube execution (to avoid extending the Cycle), we deduce:
\begin{itemize}
    \item No Cube stage can complete before $\V{k}$ starts --- otherwise, for any such completed Cube stage $\C{k'}$, we have $\V{k} + \C{k'} > \sum \C{i}$, violating the overlap constraint. Thus, the $k-1$ Vector stages preceding $\V{k}$ cannot depend on their corresponding Cube stage.
    \item No new Cube stage can start after $\V{k}$ ends, so the $n-k$ Vector stages following $\V{k}$ cannot be consumed by their Cube stages.
    \item $\V{k}$ itself can neither depend on any Cube stage nor be consumed by any Cube stage.
\end{itemize}
Consequently, at most $(k-1) + (n-k) = n-1$ internal dependency chains can be formed.

By \cref{lem:preload-count}, the minimal Preload count in this case is $(2n - 1) - (n - 1) = n$. The symmetric case ($\sum \V{i} \geq \sum \C{i}$) yields the same bound via analogous reasoning. Thus, we conclude:

\begin{lemma}\label{lem:preload_upper}
For any $n$ CV pairs satisfying the sequential dependency chain, no pipeline can universally guarantee more than $n-1$ internal dependency chains, regardless of stage durations or scheduling permutations.
\end{lemma}

\subsection{Existence of Optimal Pipeline: Formalization and Key Theorem}

We now prove that a pipeline achieving exactly $n-1$ internal dependency chains always exists under $\sum \V{i} \leq \sum \C{i}$. Without loss of generality, consider $n=3$. The six permutations of $\{\C{1}, \C{2}, \C{3}\}$ fall into two cyclic groups:
\begin{itemize}
    \item Group 1 (clockwise): $\C{1}\C{2}\C{3}$, $\C{2}\C{3}\C{1}$, $\C{3}\C{1}\C{2}$
    \item Group 2 (counter-clockwise): $\C{1}\C{3}\C{2}$, $\C{3}\C{2}\C{1}$, $\C{2}\C{1}\C{3}$
\end{itemize}

Focusing on Group 1 (see \cref{fig:inequalities}), each permutation admits two internal $\V{} \to \C{}$ dependency chains. If at least one permutation satisfies its temporal constraints, then a valid pipeline with $n-1=2$ chains exists. The same holds for Group 2, implying existence for $n=3$.

\begin{figure}[h!]
    \centering
    \includegraphics[width=\textwidth]{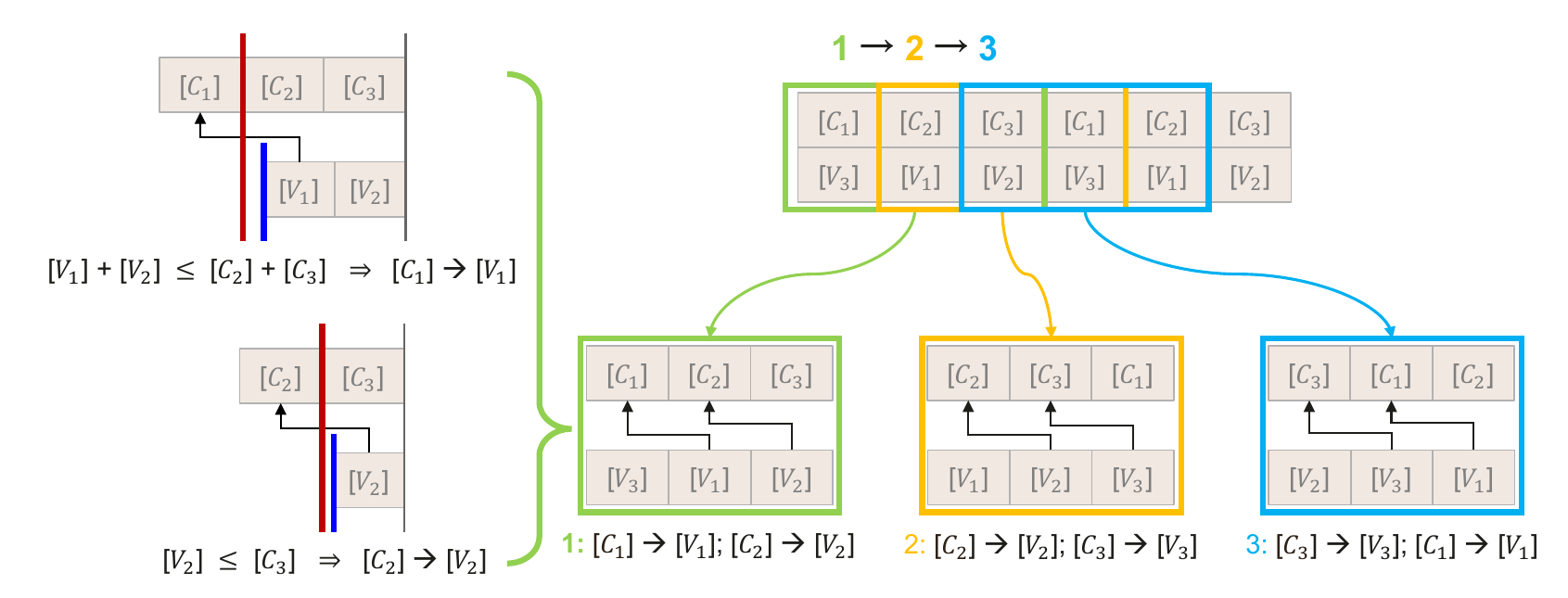}
    \caption{Cyclic permutations of $\C{1}\C{2}\C{3}$ with two internal $\V{} \to \C{}$ dependency chains and their temporal constraints.}
    \label{fig:inequalities}
\end{figure}

For permutation $\C{1}\C{2}\C{3}$, enforcing $\C{1} \to \V{1}$ and $\C{2} \to \V{2}$ requires:
\begin{equation*}
    \V{1} + \V{2} \leq \C{2} + \C{3} \quad \text{and} \quad \V{2} \leq \C{3}.
\end{equation*}
Similarly, the other two permutations in Group 1 yield:
\begin{enumerate}
    \item $\V{1} + \V{2} \leq \C{2} + \C{3}$ and $\V{2} \leq \C{3}$: enables $\C{1} \to \V{1}$, $\C{2} \to \V{2}$;
    \item $\V{2} + \V{3} \leq \C{3} + \C{1}$ and $\V{3} \leq \C{1}$: enables $\C{2} \to \V{2}$, $\C{3} \to \V{3}$;
    \item $\V{3} + \V{1} \leq \C{1} + \C{2}$ and $\V{1} \leq \C{2}$: enables $\C{3} \to \V{3}$, $\C{1} \to \V{1}$.
\end{enumerate}
We must show that at least one of these three compound conditions holds.

Generalizing to $n$ CV pairs, we formalize the existence guarantee as:

\begin{theorem}\label{thm:existence_pipeline}
Given $n$ CV pairs with dependency chain $\C{1} \to \V{1} \to \cdots \to \C{n} \to \V{n}$ and $\sum_{i=1}^n \V{i} \leq \sum_{i=1}^n \C{i}$, there exists an index $k \in \{1, \dots, n\}$ such that:
\begin{equation}\label{eq:compound_ineq}
    \bigwedge_{j=1}^{n-1} \left( \sum_{i=0}^{j-1} \V{n-k-i} \leq \sum_{i=0}^{j-1} \C{n+1-k-i} \right),
\end{equation}
where indices are interpreted cyclically modulo $n$ (i.e., $\V{n+i} \equiv \V{i}$, $\C{-i} \equiv \C{n-i}$). This guarantees the existence of a pipeline with exactly $n-1$ internal dependency chains.
\end{theorem}

\subsection{Proof of Theorem \ref{thm:existence_pipeline}}

To simplify the analysis, define an auxiliary sequence $\{a_i\}_{i=1}^n$:
\begin{equation}\label{eq:aux_sequence}
    a_i = \V{i} - \C{i+1}, \quad \text{with} \quad \C{n+1} \equiv \C{1}.
\end{equation}
The global constraint becomes:
\begin{equation}\label{eq:total_constraint}
    \sum_{i=1}^n a_i = \sum_{i=1}^n \V{i} - \sum_{i=1}^n \C{i} \leq 0.
\end{equation}

\subsubsection{Reformulation via Partial Sums}

Let $m = n - k$. Then \cref{eq:compound_ineq} is equivalent to:
\begin{equation}\label{eq:reformulated_ineq}
    \sum_{i=0}^{j-1} a_{m-i} \leq 0, \quad \forall j \in \{1, \dots, n-1\},
\end{equation}
with cyclic indexing. Thus, proving the above \Cref{thm:existence_pipeline} reduces to showing that some $m \in \{1,\dots,n\}$ satisfies \cref{eq:reformulated_ineq} for all $j$.

\subsubsection{Key Insight: Minimum Partial Sum}

Define the partial sum $F(l) = \sum_{i=1}^l a_i$ for $l = 0,1,\dots,n$, with $F(0) = 0$. By \cref{eq:total_constraint}, $F(n) \leq 0$. Let $k \in \{1,\dots,n\}$ be such that $F(k)$ is minimal.

We claim $m = k$ satisfies \cref{eq:reformulated_ineq}. Consider two cases:

\paragraph{Case 1: $1 \leq j \leq k$}
\begin{equation}\label{eq:case1_sum}
    \sum_{i=0}^{j-1} a_{k-i} = F(k) - F(k-j) \leq 0,
\end{equation}
since $F(k) \leq F(k-j)$ by minimality.

\paragraph{Case 2: $k < j \leq n-1$}
Using cyclicity $a_t = a_{t+n}$:
\begin{align}
    \sum_{i=0}^{j-1} a_{k-i} 
    &= \sum_{i=0}^{j-1} a_{k-i+n} \nonumber \\
    &= F(k+n) - F(k+n-j) \nonumber \\
    &= \left[F(n) + F(k)\right] - F(k+n-j) \label{eq:case2_F_expand} \\
    &\leq F(n) \leq 0, \label{eq:case2_result}
\end{align}
since $F(k+n-j) \geq F(k)$ and $F(n) \leq 0$.

\subsubsection{Conclusion}

In both cases, the inequality holds. Therefore, setting $m = k$ (i.e., choosing the permutation aligned with the minimum partial sum) guarantees the existence of $n-1$ internal dependency chains. For the symmetric case $\sum \V{i} \geq \sum \C{i}$, the same logic applies by reversing roles and inequalities. This completes the proof of \cref{thm:existence_pipeline}.

\end{document}